\documentclass{article}

\usepackage[preprint,nonatbib]{neurips_2024}
\RequirePackage[style      = numeric-comp,%
                giveninits=true,%
                uniquename=false,%
                uniquelist=false,%
                sorting    = nyt,%
                sortcites  = true,%
                autopunct  = true,%
                hyperref   = true,%
                abbreviate = false,%
                backref    = true,%
                backend    = biber,%
                citereset  = section]{biblatex}
\defbibheading{bibliography}[\bibname]{\section*{#1}} 
\usepackage[utf8]{inputenc} 
\usepackage[T1]{fontenc}    
\usepackage{hyperref}       
\usepackage{url}            
\usepackage{booktabs}       
\usepackage{amsfonts}       
\usepackage{nicefrac}       
\usepackage{microtype}      
\usepackage{xcolor}         
\usepackage{multirow}

\usepackage{wrapfig,framed}
\usepackage{subcaption}

\usepackage{amsthm}
\newtheorem{theorem}{Theorem}

\usepackage{amsmath,amsfonts,bm}

\usepackage{array}

\def\eqref#1{equation~\ref{#1}}

\def\1{\bm{1}}

\newtheorem{proposition}[theorem]{Proposition}

\DeclareMathAlphabet{\mathsfit}{\encodingdefault}{\sfdefault}{m}{sl}
\SetMathAlphabet{\mathsfit}{bold}{\encodingdefault}{\sfdefault}{bx}{n}
\newcommand{\tens}[1]{\bm{\mathsfit{#1}}}
\def\tA{{\tens{A}}}

\def\tK{{\tens{K}}}

\def\tX{{\tens{X}}}

\newcommand{\bfB}{{\bf B}}
\newcommand{\bfC}{{\bf C}}
\newcommand{\bfD}{{\bf D}}

\newcommand{\bfI}{{\bf I}}

\newcommand{\bfM}{{\bf M}}

\newcommand{\bfb}{{\bf b}}

\newcommand{\bfx}{{\bf x}}

\newcommand{\bfq}{{\bf q}}

\newcommand{\bfd}{{\bf d}}

\newcommand{\bfr}{{\bf r}}

\newcommand{\bfz}{{\bf z}}

\newcommand{\grad}{{\boldsymbol \nabla}}

\newcommand{\bfepsilon}{{\boldsymbol \epsilon}}
\newcommand{\bftheta}{{\boldsymbol \theta}}

\renewcommand{\grad}{{\boldsymbol \nabla\, }}

\def\mydefb#1{\expandafter\def\csname bf#1\endcsname{\mathbf{#1}}}
\def\mydefallb#1{\ifx#1\mydefallb\else\mydefb#1\expandafter\mydefallb\fi}
\mydefallb aAbBcCdDeEfFgGhHiIjJkKlLmMnNoOpPqQrRsStTuUvVwWxXyYzZ0123456789\mydefallb

\def\mydefgreek#1{\expandafter\def\csname bf#1\endcsname{\text{\boldmath$\mathbf{\csname #1\endcsname}$}}}
\def\mydefallgreek#1{\ifx\mydefallgreek#1\else\mydefgreek{#1}%
   \lowercase{\mydefgreek{#1}}\expandafter\mydefallgreek\fi}
\mydefallgreek {alpha}{Alpha}{beta}{Beta}{gamma}{Gamma}{delta}{Delta}{epsilon}{Epsilon}{varepsilon}{Varepsilon}{zeta}{Zeta}{eta}{Eta}{theta}{Theta}{iota}{Iota}{kappa}{Kappa}{lambda}{Lambda}{mu}{Mu}{nu}{Nu}{omicron}{Omicron}{pi}{Pi}{rho}{Rho}{sigma}{Sigma}{tau}{Tau}{upsilon}{Upsilon}{varphi}{phi}{Phi}{xi}{Xi}{chi}{Chi}{psi}{Psi}{omega}{Omega}\mydefallgreek

\def\mydefb#1{\expandafter\def\csname bb#1\endcsname{\mathbb{#1}}}
\def\mydefallb#1{\ifx#1\mydefallb\else\mydefb#1\expandafter\mydefallb\fi}
\mydefallb aAbBcCdDeEfFgGhHiIjJkKlLmMnNoOpPqQrRsStTuUvVwWxXyYzZ\mydefallb

\def\mydefb#1{\expandafter\def\csname cal#1\endcsname{\mathcal{#1}}}
\def\mydefallb#1{\ifx#1\mydefallb\else\mydefb#1\expandafter\mydefallb\fi}
\mydefallb aAbBcCdDeEfFgGhHiIjJkKlLmMnNoOpPqQrRsStTuUvVwWxXyYzZ\mydefallb

  \RequirePackage{listings}
  \RequirePackage{tikz}
    \usetikzlibrary{arrows}
    \usetikzlibrary{arrows.meta}
    \usetikzlibrary{backgrounds}
    \usetikzlibrary{calc}
    \usetikzlibrary{decorations.pathreplacing}
    \usetikzlibrary{tikzmark}
    \usetikzlibrary{shapes.geometric}
    \usetikzlibrary{positioning}
    \usetikzlibrary{decorations}
    \usetikzlibrary{fit}
    \usetikzlibrary{shapes}
  \RequirePackage{tikz-3dplot}
  \RequirePackage{pgfplots}
    \pgfplotsset{compat=1.16}
    \pgfdeclarelayer{background}
    \pgfdeclarelayer{foreground}
    \pgfsetlayers{background,main,foreground}
  \RequirePackage{tkz-euclide}

\definecolor{matlab1}{RGB}{0,    114,  189} 
\definecolor{matlab2}{RGB}{217,   83,   25}
\definecolor{matlab3}{RGB}{237,  177,   32}
\definecolor{matlab4}{RGB}{126,   47,  142}
\definecolor{matlab5}{RGB}{119,  172,   48}
\definecolor{matlab6}{RGB}{77,   190,  238}
\definecolor{matlab7}{RGB}{162,   20,   47}

\newcommand{\mdot} {\,\cdot\,}
\newcommand{\thf}  {\tfrac{1}{2}}                            
\DeclareMathOperator*{\argmin}{arg\,min}                   
\renewcommand{\t} {^{\top}}                                
\renewcommand{\vec}[1] {{\rm vec}(#1)}                     
\newcommand{\norm} [2][]{\left\|#2\right\|_{#1}}           

\usepackage{cleveref}

\title{Paired Autoencoders for Likelihood-free Estimation in Inverse Problems}

\author{%
  Matthias Chung\thanks{.} \\
  Department of Mathematics\\
  Emory University\\
  Atlanta, GA 30322, USA \\
  \texttt{matthias.chung@emory.edu} \\
  \And
  Emma Hart\\
  Department of Mathematics\\
  Emory University\\
  Atlanta, GA 30322, USA \\
  \texttt{ehart5@emory.edu} \\
  \And
  Julianne Chung\\
  Department of Mathematics\\
  Emory University\\
  Atlanta, GA 30322, USA \\
  \texttt{jmchung@emory.edu} \\
  \And
  Bas Peters\\
  Computational Geosciences Inc.\\
  Vancouver, BC, Canada \\
  \texttt{bas@compgeoinc.com} \\
  \And
  Eldad Haber \\
  The University of British Columbia \\
  Vancouver, BC, Canada \\
  \texttt{ehaber@eoas.ubc.ca} \\
}

\begin{document}

\maketitle

\begin{abstract}
We consider the solution of nonlinear inverse problems where the forward problem is a discretization of a partial differential equation. Such problems are notoriously difficult to solve in practice and require minimizing a combination of a data-fit term and a regularization term. The main computational bottleneck of typical algorithms is the direct estimation of the data misfit.  Therefore, likelihood-free approaches have become appealing alternatives.  Nonetheless, difficulties in generalization and limitations in accuracy have hindered their broader utility and applicability. In this work, we use a paired autoencoder framework as a likelihood-free estimator for inverse problems. We show that the use of such an architecture allows us to construct a solution efficiently and to overcome some known open problems when using likelihood-free estimators.  In particular, our framework can assess the quality of the solution and improve on it if needed. We demonstrate the viability of our approach using examples from full waveform inversion and inverse electromagnetic imaging.
\end{abstract}

\section{Introduction}
\label{sec:intro}
We consider the nonlinear inverse problem of distributed parameter estimation from measured noisy data where the forward problem is given by partial differential equations (PDEs). Such problems include, for example, the electrical impedance tomography method in medical imaging, seismic full waveform inversion (FWI) \cite{claerbout,pratt1999}, electromagnetic inversion, and direct current resistivity, which are commonly used in geophysics \cite{haberBook2014,yaoguo1994inversion}. Other applications include the retrieval of images in transport and underwater acoustic imaging \cite{sutton1979underwater}.

\paragraph{Overview on solution techniques:} Techniques to estimate such parameters require the discretization of the differential equation (the forward problem) and solving the forward problem in space and potentially in frequency or time. For 3D problems, the solution process can be time-consuming, and it requires careful numerical considerations for stability and accuracy. The forward problem is then used within an inversion procedure that typically includes some regularized data fitting process. Such a process requires the solution of the adjoint problem, which is yet another PDE. In a realistic inversion procedure, the PDE is solved hundreds or even thousands of times. Furthermore, in many inverse problems containing multiple sources and frequencies, such as the multiple source seismic problem \cite{krebs09ffw}, every forward problem may include hundreds if not thousands of sources, making every forward problem computationally expensive. As a result, PDE-based inverse problems require extensive computations which are often performed on large computer clusters.

Another difficulty that nonlinear inversion processes face is the ``correct'' choice of regularization operators, which includes a priori information. This information can reside on a highly non-convex manifold. Even if its manifold is known, incorporating those operators in an inversion process can be challenging, where the problem even without regularization is non-convex and each iteration is computationally expensive.  This can lead to further computational costs due to slow convergence. Therefore, most nonlinear inversion routines use simple convex or almost convex regularizers such as smoothness priors and total variation \cite{RudinOsherFatemi92, jc1, chung2023variable}.

\paragraph{Machine learning in inverse problems:}
Machine learning has been taking an increasing role in the solution of inverse problems (see e.g., \cite{jin2017deep, mardani2018neural, lucas2018using, adler2017solving, aggarwal2018modl, ongie2020deep, bai2020deep, eliasof2023drip, gonzalez2022solving} and references within). Roughly speaking, there are two different approaches to the use of AI techniques in inverse problems. First, in an end-to-end approach (e.g., \cite{ongie2020deep}, \cite{wu2020inversionnet}), a regularizer is trained with the forward and adjoint operators to obtain a regularization function that yields optimal solutions (in some sense). Such techniques are highly effective and draw from the wealth of optimization algorithms as well as from the use of neural networks for computer vision. While these methods work well, their one main drawback is that they use the forward and adjoint operators in the training as well as in inference. Practically, it is impossible to train on thousands of examples when each step in a stochastic gradient descent method requires hundreds of PDE solves. This is the reason why these algorithms have been limited to problems where the computation of the forward problem is computationally inexpensive, e.g., mostly linear problems. A second set of algorithms are so-called plug-and-play algorithms. In these algorithms, one learns the manifold of the model without any data. This can be done by learning the prior distribution via diffusion models \cite{cao2024survey}, or learning a parameterized representation of the model using Generative Adversarial Models (e.g., \cite{shah2018solving}) or Variational Autoencoders (e.g., \cite{goh2019solving}). By avoiding the forward and adjoint when training the AI model, they can be used generally with any inverse problem. Nonetheless, since the prior information can be highly non-convex, the resulting optimization problem is difficult to solve and is prone to converge to local minima.

\paragraph{Likelihood-free methods:}
One way to obfuscate the forward and adjoint computation is likelihood-free estimators (LFE) \cite{papamakarios2019neural, hermans2020likelihood, sainsbury2023likelihood, sainsbury2022fast} and in particular neural point estimators \cite{sainsbury2022fast}, that is, a neural network that directly returns an estimate from data. Such techniques use the pairs of model data together to either estimate the posterior directly or, in the case of a point estimate, to estimate the solution to the inverse problem directly. These methods circumvent the solution of the forward problem and thus can be thought of as attractive replacements for cases where the forward problem is computationally demanding. These techniques are natural in the context of machine learning as they involve learning an unknown map from the data space to the model space, but they suffer several shortcomings.
First, such techniques can be highly sensitive to the particular experimental setting and require non-trivial architectures. This is particularly true if the data is very different from the model, that is, the data and model ``live'' in very different functional spaces. Second, the likelihood-free aspect of the estimator which makes it so attractive is also its Achilles heel. Upon inference, it is impossible to know if the method yields an acceptable model with an acceptable data misfit. Since training the network requires that the method fits the data only on average, it is possible that the estimator obtained at inference does not fit the data. For this reason, fewer papers use this strategy in inverse problems, and they are typically tied to a particular application such as seismic imaging \cite{deng2022openfwi, zhu2023fourier} or Electrical Impedance Tomography (EIT) \cite{hamilton2018deep, fan2020solving}. These papers use similar architectures; they transform data into an encoded space and then decode this space into the model space. While this approach can work well, it has two main drawbacks. First, it is domain-specific and difficult to generalize. For example, a model that is derived for EIT cannot be used for a different problem that is recovering the same physical quantity, namely, electrical impedance. Second, one cannot combine these approaches in methods that estimate the likelihood if needed.

\paragraph{Inversion with two autoencoders:}

Various inversion networks have been considered that leverage two autoencoders, one for the input and one for the target representation, similar to our proposed paired autoencoder framework. Kun et. al. \cite{kun2015coupled} employ coupled deep autoencoders for single-image superresolution, where they work on image patches and introduce a nonlinear mapping between the latent spaces. For seismic imaging, Feng et. al. \cite{feng2023simplifying} incorporate self-supervised learning techniques to reduce the size of the supervised learning task. That is, for a given inversion task, by separately training one masked autoencoder \cite{he2021masked} for the data and another for the model, the nonlinear, large-scale supervised learning task is replaced with a smaller-scale linear task in the latent space. Although the network structures described in \cite{kun2015coupled,feng2023simplifying} are similar to the method we discuss here, there are a few key differences.  First, our approach differs in the training stage. We define a combined loss function to train both autoencoders and the latent space mappings simultaneously. In this way, our approach can be interpreted as training an encoder/decoder with regularization to ensure data-driven compression of the model and data.  Second, we develop a theory for general paired autoencoder frameworks, enabling the use of different types of autoencoders as well as different approaches for generalization. Third, we use the data decoder and model encoder for metainformation on the expected performance of the inversion (i.e., approximating data fit properties that are not possible with LFEs) and provide approaches for solution correction.

\paragraph{Main contributions:}
The goal of this paper is to propose an architecture for LFEs that is robust for different inverse problems and experimental settings, and that can be used as a stand-alone (without the computation of the likelihood) or, within a canonical estimation process of minimizing some data fitting term where one validate the results of the LFE and improves upon it if needed. To this end, we use two different autoencoders. The first one is for the model and the second one is for the data. These can be trained separately as in \cite{feng2023simplifying, kun2015coupled} or, as proposed here, in tandem. We refer to this approach as a mirrored or paired autoencoder as two autoencoders that are related to the same object are trained. We show that similar to other LFEs in inverse problems, such an approach can be used as a stand-alone LFE; however, we also show that, unlike other LFEs, our approach enables us to assess the quality of the recovered solution.  Thus, it can be used within a standard data-fitting process for further refinement. Finally, we show that our approach is versatile and can be used in a wide range of scenarios.

\paragraph{Limitations:} While our approach provides a general framework for likelihood-free surrogate inversion, there are naturally some limitations.  First, paired autoencoders are specific to the problem at hand and may require re-computation when the problem is modified (e.g., a new forward model or data acquisition modality).  Second, a set of paired autoencoders would not work for out-of-distribution data directly.  To address this, we propose inexpensive and available metrics to detect if new data are in distribution and describe an update approach to refine solutions if needed.

\section{Mathematical Preliminaries}
In this section, we lay the mathematical foundation of our approach. Let $\bfq \in \calQ$ be a discretized parameter vector in $\bbR^n$ and assume that we have a forward model which we write in general as
\begin{equation} \label{forward}
    F(\bfq) + \bfepsilon = \bfb.
\end{equation}
Here, $\bfb \in \calB$ is a data vector in $\bbR^m$, and $F:\calQ \rightarrow \calB$ is the forward mapping which is typically a discretization of some underlying PDE model. We assume that the noise is i.i.d. Gaussian, that is,  $\bfepsilon \sim \calN({\bf0}, \sigma^2 \bfI)$. Our goal is to estimate $\bfq$ given the noisy measurements $\bfb$. Since $F$ is a forward operator that may not admit a continuous inverse operator, or, in its discrete form is highly ill-conditioned, there may not be a unique solution to the \emph{inverse problem} of recovering $\bfq$ given observations $\bfb$ and operator $F$.

Canonical techniques for solving the inverse problem utilize regularized least-squares techniques that can be viewed as a maximum a posteriori estimate in a Bayesian framework \cite{taran, Tenorio2011, somersallo}. Here, an estimator $\widehat \bfq$ can be obtained by maximizing the posterior density or solving the regularized least-squares problem,
\begin{equation} \label{regopt}
    \widehat \bfq \in \argmin_{\bfq} \  \tfrac{1}{2 \sigma^2} \norm{F(\bfq) - \bfb}^2 + R(\bfq),
\end{equation}
where $\| \cdot \|$ denotes the Euclidean norm unless otherwise stated.
The first term in \Cref{regopt} is the data misfit, which is also equivalent to the negative log of the likelihood. The second term, $R$, is the regularization functional carrying prior information about $\bfq$. As stated earlier, learning the regularization operator $R$ is one way to use data-driven machine learning in the context of inverse problems \cite{afkham2021learning}.

A second approach that can be adapted is to re-parameterize $\bfq$ and then regularize the parameterized space.  In this setting, one defines $\bfq = \bfD \bfz$ where $\bfD$ is an over-complete dictionary and $\bfz$ is a vector. It has been argued \cite{eladReview,eliasof2024over} that $\bfz$ can be regularized using the $\ell^1$ norm, thus, this reparametrization of $\bfq$ replaces $R(\bfq)$ with a simple convex function \cite{newman2023image,candes2005decoding}. A more general approach is to use a parameterized decoder setting,
\begin{equation} \label{qdec}
    \bfq = D_q(\bfz_q; \bftheta_q^{\rm d}),
\end{equation}
where $\bfz_q$ denotes the latent space variables and $\bftheta_q^{\rm d}$ are the decoder parameters. This approach is sometimes referred to as a ``plug-and-play'' approach \cite{pascual2021plug} and yields the optimization problem,
\begin{equation} \label{regoptz}
    \widehat \bfz_q \in \argmin_{\bfz_q} \ \tfrac{1}{2 \sigma^2} \norm{F(D_q(\bfz_q; \bftheta_q^{\rm d})) - \bfb}^2 + R_q(\bfz_q).
\end{equation}
Note that if one specifies a Gaussian prior probability density for $\bfz_q$, i.e., $\bfz_q \sim \mathcal{N}(\mathbf{0}, \lambda^2\bfI),$ then the regularization term becomes $R_q(\bfz_q) = \tfrac{1}{\lambda^2} \norm{\bfz_q}^2$. This has motivated the use of Generative Adversarial Networks (GANs) and Variational Autoencoders (VAEs) as a way to obtain such a decoder $D_q$. Nonetheless, in our context of solving nonlinear inverse problems, computing the forward mapping and gradient requires considerable computational resources. Furthermore, the new nonlinear operator $F(D_q(\,\cdot\,;\bftheta_q^{\rm d}))$ may be even more nonlinear leading to additional computational burden and slow convergence. We thus turn to methods that at inference do not require the computation of either the forward or the adjoint operator, or seek an update in the latent space.

\section{Paired Autoencoders Framework} \label{AE}
In the following, we provide our main contribution, a new paired autoencoder architecture for likelihood-free estimation.  We discuss our proposed framework in \Cref{MAE} and describe how to train paired autoencoders in \Cref{training}. For cases where the estimator is not sufficient or the data is very different from the training set, we describe in \Cref{inference} various techniques to detect out of distribution data and subsequently refine solutions to improve inference via updating in the latent space.

\subsection{Paired autoencoders for likelihood-free estimation} \label{MAE}
Our goal is to build an LFE for $\bfq$ given some noisy observations $\bfb$. Before we focus on this goal, let us consider two simple autoencoders. For $\bfq \in \calQ$, let $E_{q}$ and $D_{q}$ denote the encoder and decoder, which depend on network parameters $\bftheta_q^{\rm e}$ and $\bftheta_q^{\rm d}$, respectively. Similarly, we denote $E_{b}$ and $D_{b}$ to be the encoder and decoder for $\bfb \in \calB$ which in turn depend on network parameters $\bftheta_b^{\rm e}$ and $\bftheta_b^{\rm d}$, respectively. In principle, one may use any encoding/decoding architecture.

The latent space vectors $\bfz_q \in \calZ_q$ and $\bfz_b \in \calZ_b$ and the representations of $\bfq$ and $\bfb$ through the autoencoders $\widetilde \bfq$ and $\widetilde \bfb$ are given as,
\begin{subequations} \label{allEq}
\begin{align} \label{encoders}
    \bfz_q &= E_q(\bfq;\, \bftheta_q^{\rm e}), & \bfz_b &=  E_b(\bfb;\, \bftheta_b^{\rm e}),      \\ \label{decoders}
    \widetilde \bfq &= D_q(\bfz_q; \bftheta_q^{\rm d}), & \widetilde \bfb &=  D_b(\bfz_b; \bftheta_b^{\rm d}) .
\end{align}
\end{subequations}

\begin{figure}
    \centering





\begin{tikzpicture}

	\node[fill=matlab1!50, minimum width=2.0cm, minimum height=0.5cm] (x) at (0,0) {$\bfq$};
    \draw[fill=matlab2!50,draw=none]
             ([yshift=-0.2cm,xshift=0.0cm]x.south west) -- ([yshift=-1.5cm,xshift=0.5cm]x.south west) -- ([yshift=-1.5cm,xshift=-0.5cm]x.south east) -- ([yshift=-0.2cm,xshift=0.0cm]x.south east) -- cycle node (e) at ([yshift=-0.75cm,xshift=0.0cm]x.south) {$\begin{matrix}\mbox{encoder} \\[-0.0ex] E_{q} \end{matrix}$}; 
               
    \node[fill=matlab3!50, minimum width=1cm, minimum height=0.5cm] (zx) at (0,-2.3) {$\bfz_q$};    
    \draw[fill=matlab4!50,draw=none]
             ([yshift=-0.2cm,xshift=0.0cm]zx.south west) -- ([yshift=-1.5cm,xshift=-0.5cm]zx.south west) -- ([yshift=-1.5cm,xshift=0.5cm]zx.south east) -- ([yshift=-0.2cm,xshift=0.0cm]zx.south east) -- cycle node (d) at ([yshift=-0.75cm,xshift=0.0cm]zx.south) {$\begin{matrix}\mbox{decoder} \\[-0.0ex] D_{q} \end{matrix}$};  
    \node[fill=matlab1!50, minimum width=2.0cm, minimum height=0.5cm] (xt) at (0,-4.5) {$\bfq$};      

    \node[fill=matlab5!50, minimum width=2.5cm, minimum height=0.5cm] (b) at (4,0) {$\bfb$};
    \draw[fill=matlab6!50,draw=none]
             ([yshift=-0.2cm,xshift=0.0cm]b.south west) -- ([yshift=-1.5cm,xshift=0.5cm]b.south west) -- ([yshift=-1.5cm,xshift=-0.5cm]b.south east) -- ([yshift=-0.2cm,xshift=0.0cm]b.south east) -- cycle node (e) at ([yshift=-0.75cm,xshift=0.0cm]b.south) {$\begin{matrix}\mbox{encoder} \\[-0.0ex] E_{b} \end{matrix}$}; 
               
    \node[fill=matlab3!50, minimum width=1.5cm, minimum height=0.5cm] (zb) at (4,-2.3) {$\bfz_b$};    
    \draw[fill=matlab7!50,draw=none]
             ([yshift=-0.2cm,xshift=0.0cm]zb.south west) -- ([yshift=-1.5cm,xshift=-0.5cm]zb.south west) -- ([yshift=-1.5cm,xshift=0.5cm]zb.south east) -- ([yshift=-0.2cm,xshift=0.0cm]zb.south east) -- cycle node (d) at ([yshift=-0.75cm,xshift=0.0cm]zb.south) {$\begin{matrix}\mbox{decoder} \\[-0.0ex] D_{b} \end{matrix}$};  
    \node[fill=matlab5!50, minimum width=2.5cm, minimum height=0.5cm] (bt) at (4,-4.5) {$\bfb$}; 

\begin{scope}[-latex,shorten >=9pt,shorten <=9pt,line width=5pt]
    \draw[matlab2!50!matlab3]  ([yshift=-0.2cm]zb.west) to ([yshift=-0.2cm]zx.east);
    \draw[matlab1!50!matlab2] ([yshift=0.2cm]zx.east) to ([yshift=0.2cm]zb.west);    
\end{scope}
\node[matlab1!50!matlab2] (m_forward) at (2.0,-1.7) {$\bfM$};
\node[matlab2!50!matlab3] (m_inverse) at (2.0,-2.9) {$\bfM^{\dagger}$};

\begin{scope}[-latex,shorten >=9pt,shorten <=9pt,line width=5pt]
    \draw[matlab1!50!matlab2] ([yshift=-0.0cm]x.east) to ([yshift=-0.0cm]b.west);
    \draw[matlab2!50!matlab3] ([yshift=0.0cm]bt.west) to ([yshift=0.0cm]xt.east);
\end{scope}

\node[matlab1!50!matlab2] (f_forward) at (2.0,0.4) {$F$};
\node[matlab2!50!matlab3] (f_inverse) at (2.0,-4.9) {$F^{\dagger}$};
     
\end{tikzpicture}
    \caption{Architecture of our proposed paired autoencoder framework. Two autoencoders and a (linear) mapping between the latent spaces are simultaneously learned, thereby creating a \emph{paired} system.} \label{fig:pair}
\end{figure}
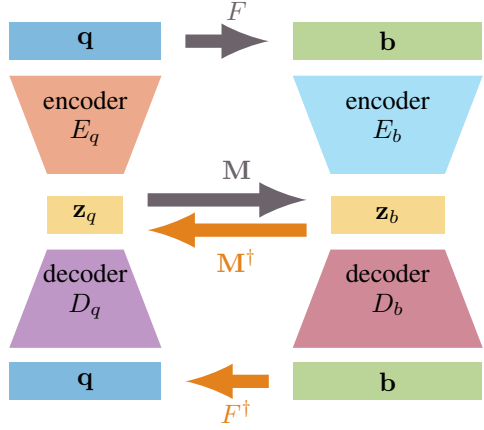
This is the usual setting of encoded and decoded spaces for autoencoders, where the aim is to have $\bfq \approx \widetilde \bfq$ and $\bfb \approx \widetilde \bfb$.  In the architecture of our proposed framework in \Cref{fig:pair}, this corresponds to the autoencoders on the left and right, respectively. In the following, for ease of notation, we omit the dependency of $E_q, E_b, D_q, D_b$ on their corresponding network parameters $\bftheta_q^{\rm e}, \bftheta_b^{\rm e}, \bftheta_q^{\rm d}, \bftheta_b^{\rm d}$ when appropriate.

Next, we consider the mapping between the encoded subspaces $\calZ_q$ and $\calZ_b$. In general, there does not exist a simple mapping between the encoded spaces. However, since we learn the mappings, we can require the resulting relation between the subspaces to be simple. In particular, we may impose the mapping to be linear or even the identity. This leads to the following set of equations,
\begin{equation}
\label{connections}
    \bfz_q \approx  \bfM^\dagger \bfz_b\quad \mbox{and} \quad        \bfz_b \approx  \bfM \bfz_q,
\end{equation}
where the combination of \Cref{allEq,connections} conveys the demands that (1) the encoders can successfully encode into the latent spaces, (2) the decoders can successfully decode from the latent spaces, and (3) the two autoencoders are \emph{paired} via a linear relationship between the two latent spaces that can be learned or assigned.
We remark that a linear correlation between latent spaces was observed in \cite{feng2023simplifying}, but since all elements were trained independently (e.g., the masked autoencoders and the linear mapping), various components were ignored including the coupling of the autoencoders and the encoder $E_q$ and decoder $D_b$. On the contrary, our framework exploits all components of the framework.  The architecture of our paired autoencoder framework is summarized in \Cref{fig:pair}.

Assuming that we can train the system such that \Cref{allEq,connections,connections} approximately hold, then we obtain an immediate likelihood-free estimator for $\bfq$ given $\bfd$. More specifically, we define an approximate likelihood-free surrogate forward and inverse mapping respectively as,
\begin{equation} \label{lfe}
    F(\bfq) \approx D_b(\bfM E_q(\bfq)) \quad \mbox{and} \quad  F^\dagger(\bfb) \approx D_q(\bfM^{\dagger} E_b(\bfb)).
\end{equation}

\subsection{Training the paired autoencoders} \label{training}

The process of training the paired autoencoder framework involves finding
parameters $\bftheta = [\bftheta_q^{\rm e}; \bftheta_b^{\rm e}; \bftheta_q^{\rm d}; \bftheta_b^{\rm d}; \vec{\bfM}; \vec{\bfM^\dagger}]\t$ that minimize the expected loss,
\begin{equation} \label{eq:loss1}
L(\bftheta) = \bbE \
    \thf \norm{D_b(E_b(\bfb; \bftheta_b^{\rm e});\bftheta_b^{\rm d}) - \bfb}^2 +
    \thf \norm{D_q(E_q(\bfq; \bftheta_q^{\rm e});\bftheta_q^{\rm d}) - \bfq}^2 + \thf S(\bfq, \bfb; \bftheta).
\end{equation}
The first two terms are the standard autoencoder losses for $\bfb$ and $\bfq$, respectively. The third term $S$ couples the autoencoders through the encoded spaces. Here, one has many choices; for instance,
\begin{align}
\label{eq:B}
    S(\bfq, \bfb; \bftheta) &=  \norm{ \bfb - D_b(\bfM E_q(\bfq; \bftheta_q^{\rm e}); \bftheta_b^{\rm d} )}^2, & &\mbox{loss in data space $\calB$, or}\\
     \label{eq:Q}  S(\bfq, \bfb; \bftheta) & =  \norm{\bfq - D_q(\bfM^\dagger E_b(\bfb; \bftheta_b^{\rm e});\bftheta_q^{\rm d})}^2,  & &\mbox{loss in parameter space $\calQ$}.
\end{align}

Depending on the objective of the application at hand, one may utilize \Cref{eq:B} for learning $\bfM$ and \Cref{eq:Q} for learning $\bfM^\dagger$. We empirically observe that for the cases presented here, both \Cref{eq:B,eq:Q} can be used with fixed mappings, e.g., $\bfM  = \bfM^\dagger =\bfI$ is sufficient.

In a data-driven approach, we have representative input-target pairs $\left\{\bfb_i,\,\bfq_i\right\}_{i = 1}^N$. These realizations $\bfq_i$ can be obtained by using, for example, the geo-statistical package \cite{diggle1998model, mariethoz2014multiple} for geoscientific applications or randomized Shepp-Logan phantoms for medical computer tomography \cite{randomizedSheppLogan}. Let us assume that we have a method that can compute the forward mapping $F(\bfq)$, e.g., by utilizing some PDE solver. In this case, we may sample \emph{coupled} pairs $\left\{\bfb_i, \bfq_i \right\}$ with $\bfb_i = F(\bfq_i) + \bfepsilon_i$ where $\bfepsilon_i \sim \calN({\bf0}, \sigma^2\bfI)$ for $i = 1,\ldots, N$. Then paired autoencoders can be trained by approximating the expected loss \Cref{eq:loss1} by its empirical loss, i.e.,
\begin{equation} \label{eq:loss2}
L_N(\bftheta) = \tfrac{1}{2N} \sum_{i = 1}^N\Big(
    \norm{D_b(E_b(\bfb_i; \bftheta_b^{\rm e});\bftheta_b^{\rm d}) - \bfb_i}^2 +
    \norm{D_q(E_q(\bfq_i; \bftheta_q^{\rm e});\bftheta_q^{\rm d}) - \bfq_i}^2 +
    S(\bfq_i, \bfb_i; \bftheta) \Big),
\end{equation}
and e.g., using a stochastic optimization algorithm to solve $\min_\bftheta L_N(\bftheta).$

\subsection{Inference via regularization in the latent space}
\label{inference}
Contrary to previous work on encoder-decoder networks where the outcome is solely the reconstruction, the paired autoencoder framework provides new opportunities for further improvement and refinement of the solution, as well as broader applicability.

We begin with a discussion of metrics that are computationally feasible/available from the paired autoencoder framework and describe how these can be used to indicate reconstruction quality. Assume that a paired auto-encoder framework is trained, giving six different networks at hand: two encoders, two decoders, and two mappings from one encoded space to the other.  That is, we have $E_b, D_b, E_q, D_q, \bfM,$ and $\bfM^{\dagger}$. Now, the simplest way to use these networks to recover a solution $\widehat \bfq$ given a new observation data vector $\bfb$ is to compute the LFE,
\begin{equation} \label{qhat}
    \widehat \bfq = D_q (\bfM^\dagger E_b(\bfb)).
\end{equation}
One approach to determine whether $\widehat \bfq$ is a sufficient approximation to $\bfq$ is to apply the original forward model $F$ and verify if the fit to the data is sufficiently met. That is, a standard approach is to compute the residual,
\begin{equation}
\label{eq:residual}
    \bfr = F\left(\widehat \bfq \right) - \bfb.
\end{equation}
If the norm of the residual is sufficiently small, then we declare $\widehat \bfq$ as a plausible solution.
However, as discussed in the introduction, a main drawback of LFEs is that it is impossible to verify their data fit properties (\Cref{eq:residual}) without computing the forward problem which can be prohibitively expensive.

Nonetheless, with the paired autoencoder framework, there are further available metrics that are inexpensive to compute and can give an indication of the quality of the reconstruction. In particular, we focus on two such metrics that are readily available,
the relative residual estimate (RRE) and the
recovered model autoencoder (RMA)
\begin{equation}
\label{eq:metrics}
{\rm RRE}(\bfb) =\|\bfb - D_b(\bfM E_q( \widehat \bfq)) \|/\|\bfb\| \quad  {\rm and} \quad
{\rm RMA}(\widehat\bfq) =\|\widehat \bfq - D_q(E_q(\widehat \bfq))\|/\|\widehat\bfq\|.
\end{equation}

These metrics assess the quality of the reconstruction in the data space (RRE) and the recovery in the model autoencoder (RMA), both using the LFE $\widehat \bfq$. While our paired autoencoder does not directly yield a probability density, the distribution of these metrics $p({\rm RRE}(\bfb),{\rm RMA}(\widehat\bfq))$ is at our disposal. For out-of-distribution (OOD) detection (e.g., \cite[Ch. 18]{murphy2023probabilistic},\cite{yang2021generalized}) without any OOD samples, we may now utilize $p({\rm RRE}(\bfb),{\rm RMA}(\widehat\bfq))$ to examine new data. If these metrics for the incoming data lie within high-probability regions, we may trust the data and the recovery; however, if the probability is small, we conclude that further investigation is required. This approach is a variation of the ideas of density thresholding \cite{bishop1994novelty,nalisnick2018do,an2015variational} and methods based on reconstruction errors \cite{hawkins2002outlier,ChenEnsembleOutlier,8999265}. Although it is known that these thresholding approaches have limitations, we show that they have merit within the paired autoencoder framework.

For solutions that require further investigation (e.g., the data fit is not small or the RRE and RMA indicate OOD), we can use the encoded subspace for correction.  We call this a \textit{latent-space-inversion} (\textit{LSI}) approach since it operates in the latent model space. Given $\bfz^{\star} = \bfM^{\dag}E_b(\bfb)$ we solve for
\begin{equation} \label{eq:fullLatentInverse}
    \bfz_{\rm lsi} \in \argmin_{\bfz \in \calZ_q} \ \thf \norm{F\left(D_q(\bfz)\right) - \bfb}^2 + \tfrac{\alpha}{2} \norm{\bfz - \bfz^{\star}}^2,
\end{equation}
where $\alpha>0$ is an appropriate regularization parameter, and the solution is given by $\bfq_{\rm lsi} = D_q(\bfz_{\rm lsi})$.

This approach is similar to methods that use the encoded space to solve the inverse problem, presented in \Cref{regoptz} with one important difference. We include a Tikhonov regularization for $\bfz$ and initialize the solution with $\bfz^{\star}$ which is typically close to the solution. Thus, even though we are using the standard ``machinery'' of forwards and adjoints to fit the data in \Cref{eq:fullLatentInverse}, our likelihood-free network provides us with a good initial guess and an encoded space to work in, which allows for quicker convergence.

\section{Theoretical Analysis}
\label{theory}
Assuming that a paired autoencoder framework is trained, giving $E_b, D_b, E_q, D_q, \bfM,$ and $\bfM^{\dagger}$, additional assumptions allow us to bound both the residual and model errors ($\| F(\widehat\bfq)-\bfb\|$ and $\| \widehat\bfq-\bfq \|$).  This distinguishes our approach from many other LFEs; given new data that we wish to invert, we have information through these bounds about residual error without evaluating the forward model $F$. Proposition \ref{prop:residual_i} provides residual bounds using either an assumption on the trained inversion error or an assumption on the forward surrogate error. \Cref{thm:error} provides a bound for the model error using assumptions on the forward surrogate and autoencoder errors.  Formal statements and proofs are included in the appendix, see \Cref{sub:bounds}.

\section{Numerical Examples}\label{sec:numerics}
We provide results for two different example applications: seismic inversion (\Cref{sec:numerics}) and electromagnetic inversion (see appendix, \Cref{sub:electro}).  Both problems require the solution of a challenging nonlinear inverse problem. We train our paired autoencoders by minimizing \Cref{eq:loss2} with coupling terms \Cref{eq:B} and \Cref{eq:Q} and with $\bfM = \bfM^\dagger = \bfI.$ For a new observation $\bfb$, one may compute the LFE, which is given by \Cref{qhat}, where $\widehat \bfq = D_q(\widehat \bfz)$ with $\widehat  \bfz =  \bfM^\dagger E_b(\bfb)$ being the latent space representation of the LFE. Our proposed approach enables refinement of the LFE via inference with regularization in the latent space, i.e., LSI by solving \Cref{eq:fullLatentInverse} with $\bfz^{\star} = \widehat \bfz$. The regularization term ensures that the optimization variables $\bfz$ stay close, in $\ell^2$ sense, to the latent variables that the network encountered while training, thereby ensuring that the decoded model is realistic. We can additionally use $\widehat \bfz$ as an initial guess for the optimization algorithm, which we denote as warm-start.

For numerical comparisons, we consider Basic Inversion (BI), which corresponds to solving \Cref{regopt} with no regularization.  We consider two initial guesses: a typical guess $\bfq_0^{\rm basic}$ and the LFE $\widehat \bfq$.  We also compare to an LSI with $\alpha=0$ and initial guess $\bfz_{\bfq_0^{\rm basic}} = E_q(\bfq_0^{\rm basic})$. This approach is akin to \Cref{regoptz}, a nonlinear plug-and-play approach where the decoder serves as the parameterization that regularizes the inversion. Here one still requires an initial model if the decoder is learned from encoding and decoding example models. In contrast, our proposed approach starts from an initial guess in the latent space, provided by the paired autoencoders. All inference approaches use a gradient-based method \cite{kingma2014adam}.

The goal of seismic inversion is to reconstruct the medium parameters, such as the spatially dependent acoustic velocity $q(x)$, from observations of the space and time-dependent wavefield, $u(x,t)$, at receiver locations. Here, we consider wavefields generated by $j = 1,2,\ldots,n_s$ controlled sources $s(x,t)_j$. The above quantities are connected by the acoustic wave equations,
\begin{equation}\label{wave}
    \grad^2 u(x,t)_{j}  - \frac{1}{q(x)} \partial_{tt} u(x,t)_{j} =  s(x,t)_j.
\end{equation}
The canonical seismic full-waveform inversion problem solves the nonconvex problem stated in \Cref{regopt}. In this case, the forward model solves the wave equation given $q(x)$ and $s(x,t)$, followed by sampling the wavefield $u(x,t)$ at the receiver locations.

It is well known that the primary challenge in FWI is obtaining a sufficiently good starting model \cite{haber2014computational}. While Tikhonov or total-variation regularizers help mitigate mundane problems related to random data noise, poor source and receiver sampling, or incorrect assumptions about the source function, they do not generally solve issues related to poor starting models \cite{esser2018total,peters2019projection}. Various learned regularizers and parameterizations have been developed \cite{haten, he2021reparameterized, zhu2022integrating, TaufikLearned, 10274489}, but they, too, need to solve the inverse problem from scratch. Other approaches learn diffusion models \cite{10328845} or flows \cite{siahkoohi2022wave, yin2024wise} that start from reasonably accurate initial models or require physics-based imaging as a prerequisite.

Consider an acoustic velocity model $\bfq \in \mathbb{R}^{n_z \times n_x}$ and corresponding data $\bfb \in \mathbb{R}^{n_s \times n_r \times n_t}$ where $n_z$ and $n_x$ are the number of grid points in depth and lateral directions, respectively, and for the data, $n_s$, $n_r$ and $n_t$ denote the number of sources, receivers, and time samples respectively. A velocity model and corresponding data for one source are shown in \Cref{fig:seismic_data_model}.

\begin{figure}
    \centering
     \includegraphics[width=0.35\textwidth]{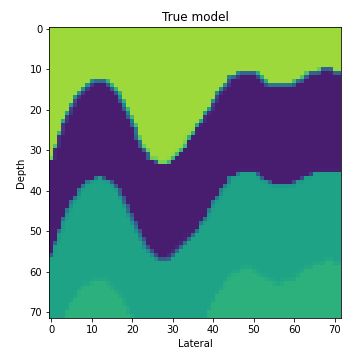}
    \includegraphics[width=0.35\textwidth]{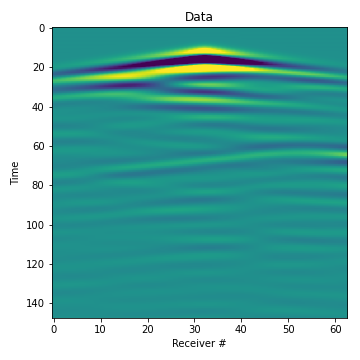}
    \caption{Seismic Inversion Example: An acoustic velocity model and the corresponding data for a single source. There are $30$ sources in total.} \label{fig:seismic_data_model}
\end{figure}

We generated about 18,000 training and 5,000 validation examples, taken from the OpenFWI database \cite{deng2022openfwi}. Wavefield data was generated using DeepWave \cite{richardson_alan_2023}. Details of the network design are given in \Cref{sub:seismicdetails}.
Once the paired autoencoders are trained, we consider both BI and LSI reconstructions with two initial guesses: first, a typical laterally invariant guess of the velocity model for $\bfq_0^{\rm basic}$, and second, the LFE $\widehat \bfq$. Both are shown in the top row of figures in \Cref{fig:seismicimages}.

\Cref{tab:seismic_quant_results} displays quantitative results of data misfit and model error for each of the four inversion approaches, averaged over $600$ samples. It is clear and expected that BI with $\bfq_0^{\rm basic}$ performs poorly in terms of data fitting and model estimation. The other three methods manage to fit the data more accurately. Note that data fitting is just a vehicle to arrive at a velocity model and that the two latent-space inversion approaches utilize highly informative regularization provided by the decoder. In terms of model error, our proposed approach (LSI with $\widehat \bfz$) achieves the best results.

\begin{table}
    \caption{Seismic Inversion. Means and (standard deviations) of the data misfit and model errors, calculated over $600$ examples. All errors are in terms of relative $\ell^2$ error.\\}
    \label{tab:seismic_quant_results}
    \centering
    \resizebox{\textwidth}{!}{%
    \begin{tabular}{llcccc}
    \toprule

&\multirow{2}{*}{\textbf{Initial Guess}}&\multicolumn{2}{c}{\textbf{Data Misfit}} & \multicolumn{2}{c}{\textbf{Model Error}} \\
&  & Initial & Final & Initial & Final \\      \hline
\multirow{2}{*}{BI}& basic start, $\bfq_0^{\rm basic}$  & 0.410 (0.227) & 0.123 (0.107) & 0.157 (0.056) & 0.160 (0.060)\\
&warm start, $\widehat\bfq$   & 0.306 (0.158) & 0.045 (0.034) & 0.140 (0.055) & 0.138 (0.056)\\ \hline
\multirow{2}{*}{LSI}&basic start, $\bfz_{\bfq_0^{\rm basic}}$, $\alpha=0$  & 0.410 (0.227)& 0.056 (0.041) & 0.157 (0.056) & 0.127 (0.069)\\
&warm start, $\widehat\bfz, \alpha = 1 $ (ours) & 0.306 (0.158) & 0.059 (0.028) & 0.140 (0.055) & \textbf{0.108} (0.047)\\
    \bottomrule
    \end{tabular}}
\end{table}

\begin{figure}
\begin{tabular}{cc|cc} \hline
\multicolumn{2}{c|}{BI} & \multicolumn{2}{c}{LSI} \\ \hline
basic start & warm start & basic start & warm start
\\ \hline
        \includegraphics[width=0.22\textwidth]{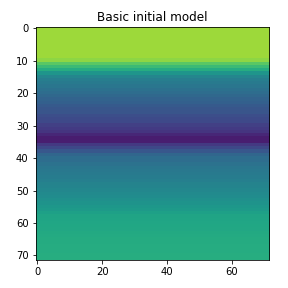}&
        \includegraphics[width=0.22\textwidth]{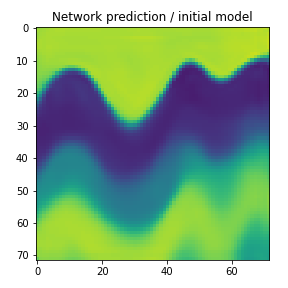}&
        \includegraphics[width=0.22\textwidth]{figures/FWI/basic_starting.png}&
        \includegraphics[width=0.22\textwidth]{figures/FWI/network_prediction.png}\\
        \includegraphics[width=0.22\textwidth]{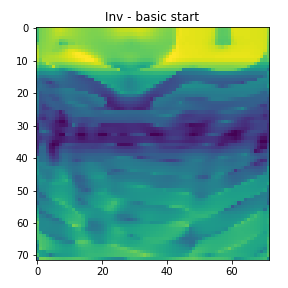}&
        \includegraphics[width=0.22\textwidth]{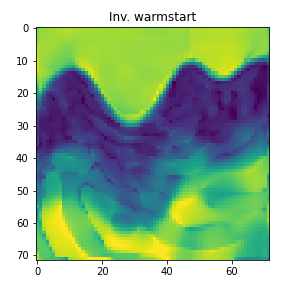}&
        \includegraphics[width=0.22\textwidth]{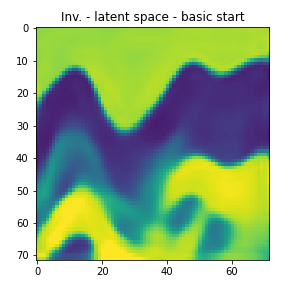}&
        \includegraphics[width=0.22\textwidth]{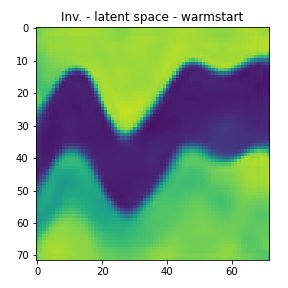}\\
        \includegraphics[width=0.22\textwidth]{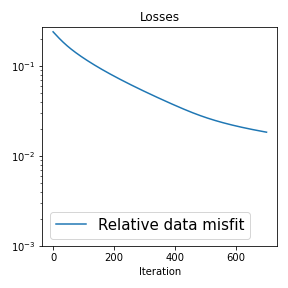}&
        \includegraphics[width=0.22\textwidth]{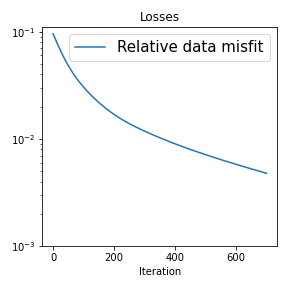}&
        \includegraphics[width=0.22\textwidth]{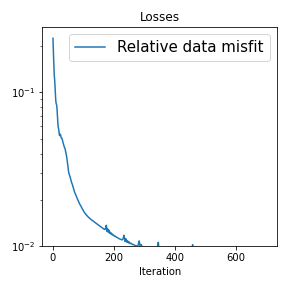}&
        \includegraphics[width=0.22\textwidth]{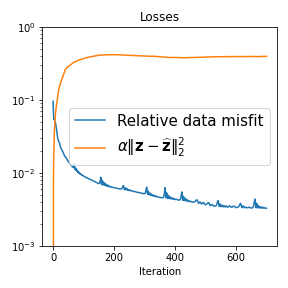}
    \end{tabular}
    \caption{Seismic reconstructions for BI and LSI, with the basic initial guess and the warm start, are provided in the second row of figures. Both BI approaches suffer from `wavefront' like artifacts.}
\label{fig:seismicimages}
\end{figure}

Seismic reconstructions obtained after $700$ iterations are provided in \Cref{fig:seismicimages}, along with the losses per iteration.  For our approach, we also provide the value of the regularization term at each iteration, highlighting the importance of the regularizer in \Cref{eq:fullLatentInverse}. In \Cref{fig:seismic_errors} we provide the relative model errors per iteration of the four considered approaches. Although these model errors are not observable in practice and cannot be used to determine a stopping criterion for optimization, they provide convergence insight. We observe that both unregularized (BI) approaches exhibit an increased error slightly after early improvements. The main difference between the LSI methods is that the warm-started (proposed) approach results in a stable and lower model error, while the generic decoder parameterization approach arrives at a realistic velocity model, but with a much higher model error.

Next, we present results according to the discussion in \Cref{inference}. We use a different dataset from the OpenFWI database \cite{deng2022openfwi} than the one used for training and validation, and we assess the ability of
\begin{figure}
    \centering
    \includegraphics[width=0.65\linewidth]{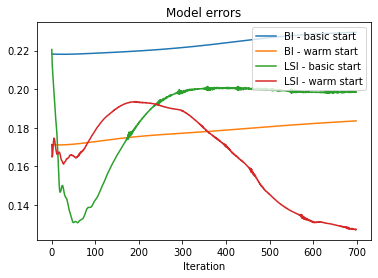}
    \caption{Relative model errors per iteration for the four approaches for seismic inversion. }
    \label{fig:seismic_errors}
\end{figure}
RRE and RMA metrics from \Cref{eq:metrics} to detect if new data is OOD. Note that these metrics are readily available, cheap to compute, and represent a novelty as they are not available for other approaches (e.g., those that train two autoencoders independently).

\Cref{fig:seismic_OOD_test} displays the validation dataset density in terms of the RRE and RMA.
These values are provided for the OOD dataset using scattered points, where the color corresponds to the reconstruction error (not available in practice).  \Cref{fig:seismic_OOD_test} also shows the LFE for two highlighted data points: one with a high density and low error and one with a low density and high error. Since the OOD data do not show significant overlap with the higher-density areas, this example suggests that we can detect OOD data and predict LFE performance using these metrics. The reconstruction error on OOD data is generally high, so detection is the primary objective.

\begin{figure}[h]
    \centering
    \includegraphics[width=0.98\linewidth]{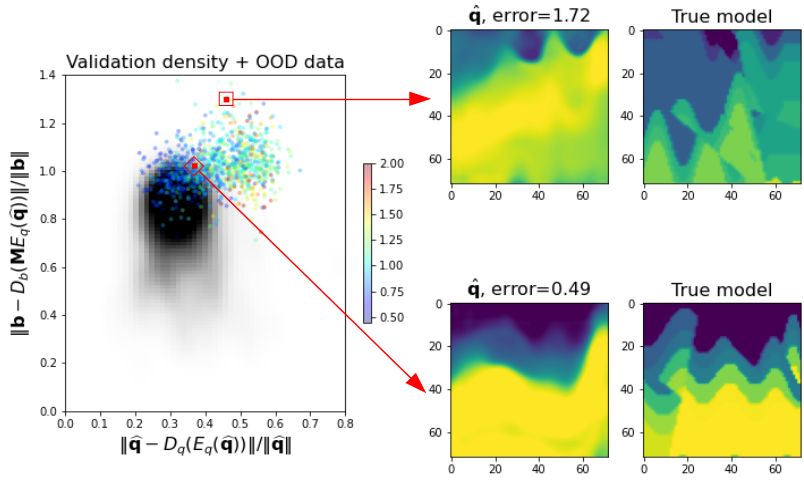}
    \caption{The density in terms of RRE and RMA for the validation set is plotted in gray-scale (with large densities in black and low densities in white).
    Scattered points correspond to metrics obtained from OOD data and are color-coded based on reconstruction error $\|\widehat \bfq - \bfq \|$. Two model predictions are provided from the paired autoencoder.}
    \label{fig:seismic_OOD_test}
\end{figure}

\section{Conclusions}
We present paired autoencoders for data-driven inversion of large-scale, ill-posed inverse problems arising from discretized PDEs. Our likelihood-free estimator addresses key limitations of other LFEs.  We describe an approach to refine initial inversion guesses by utilizing latent representations. Each of the six trained networks
can be used to provide surrogates, regularize iterative methods, and bound errors/residuals. Numerical results show that this method can outperform other techniques, providing more accurate inversions with fewer function evaluations by leveraging data.

\paragraph{Acknowledgement:} This work was partially supported by the National Science Foundation (NSF) under grant  DMS-2152661 (M. Chung), DMS-2341843 (J. Chung).  This material is based upon work partially supported by the U.S. Department of Energy, Office of Science, Office of Advanced Scientific Computing Research, Department of Energy Computational Science Graduate Fellowship under Award Number DE-SC0024386 (E. Hart).


\printbibliography

\section{Appendix}
\label{appendix}

Here, we include theoretical bounds for the residual norm and reconstruction error in \Cref{sub:bounds}.  These results provide insight regarding the described metrics for new data being ``in distribution''.  Then in \Cref{sub:electro} we provide additional numerical results for the inverse electromagnetic problem, and in \Cref{sub:seismicdetails} we describe specifics of the networks used.

\subsection{Theoretical bounds}

\label{sub:bounds}
We develop various bounds for the paired autoencoder networks, using $\bfq\in\calQ$ to denote the true model and $\bfb \in \calB$ for the observed data, with
$$\bfb = F(\bfq)+\bfepsilon.$$ Assume that a paired autoencoder framework is trained, giving $E_b, D_b, E_q, D_q, \bfM,$ and $\bfM^{\dagger}$.  For the remainder of this section, we additionally denote the following inverse/forward approximations,
\begin{equation*}
    \widehat\bfq=D_q (\bfM^\dagger E_b(\bfb)), \quad
    \widehat{\bfb}=D_b (\bfM E_q (\widehat\bfq)),\quad \text{ and } \quad \bar\bfb=D_b (\bfM E_q( \bfq)).
\end{equation*}

For problems where the forward model is computationally expensive, evaluating $F(\widehat \bfq)$ can be cumbersome.  The paired autoencoder framework provides an alternative way to check if an LFE is good without having to fit the data.  Metrics such as \Cref{eq:metrics} are readily available from the paired autoencoder framework, they are cheap to compute, and they have proven to be relevant in practice (see \Cref{fig:seismic_OOD_test}).  Next, we will show that they are also theoretically motivated, as they appear in bounds for the residual and model error.

\begin{proposition} \label{prop:residual_i}
Let $F:\calQ \to \calB$ be Lipschitz continuous on the metric spaces $(\calQ,\norm{\mdot}$) and $(\calB,\norm{\mdot})$ with Lipschitz constant $L\geq 0$. Further assume there exists an $\varepsilon_q\geq 0$ such that  $\norm{\widehat\bfq - \bfq}\leq \varepsilon_q$ for all $\bfq \in \calQ$.
Then $\norm{F(\widehat \bfq) - \bfb}\leq L \varepsilon_q + \norm{\bfepsilon}$ and $\|F(\widehat \bfq) - \bfb \| \leq \|F(\widehat \bfq) - D_b (\bfM E_q(\widehat \bfq))\| +\| \widehat{\bfb} - \bfb \|$.
\end{proposition}

\begin{proof}
By triangular inequality we get
\begin{align*}
\norm{F(\widehat \bfq) - \bfb} = \norm{F(\widehat \bfq) - F(\bfq) + F(\bfq)  - \bfb} \leq \norm{F(\widehat \bfq) - F(\bfq)} +\norm{F(\bfq)  - \bfb} \leq L \epsilon_q + \norm{\bfepsilon}.
\end{align*}
For the second statement we have
\begin{align*}\label{eq:boundqq}
\norm{F(\widehat \bfq) - \bfb} & = \norm{F(\widehat \bfq) - D_b (\bfM E_q(\widehat \bfq)) + D_b(\bfM E_q(\widehat \bfq)) - \bfb} \\
&\leq \norm{F(\widehat \bfq) - D_b (\bfM E_q(\widehat \bfq))} + \| \widehat{\bfb} - \bfb\|.
\end{align*}
\end{proof}
We remark that the two bounds give insight regarding the performance of the paired autoencoder framework, for a new sample.  For the first bound, we do not have access to the true model $\bfq$, so we must make an assumption that the error $\epsilon_q$ is small, i.e., the trained inversion is performing as expected. In that case, the bound is dominated by the last term $\|\bfepsilon\|$, which is just the amount of noise or measurement error in the data.

The second statement avoids the assumption on $\|\widehat\bfq -\bfq\|$, but for tight bounds assumes $D_b (\bfM E_q(\mdot))$ is a good surrogate for the forward model $F(\mdot)$ for any $\widehat\bfq \in \calQ$.
If the forward surrogate is good (i.e., the first term vanishes), the bound is dominated by  $\| \widehat{\bfb} - \bfb \|$, which is exactly the RRE defined in \Cref{eq:metrics}, up to a normalizing constant. Recall that we employed the RRE metric for detecting OOD samples.

Hence, if the forward surrogate is good and the RRE is small, we can expect the norm of the residual $\norm{F(\widehat \bfq) - \bfb}$ to be small, justifying our bound. Conversely, if the RRE is large, the residual bound may be large, which merits further investigation.

The residual norm is critical for determining fit to data, but oftentimes in inverse problems, we are also interested in the model error $||\widehat\bfq-\bfq||$. Although not computable in practice, it can be insightful to have an upper bound, which we provide next.

\begin{proposition} \label{prop:error1}
Let $D_q (E_q(\mdot)):\calQ \to \calQ$ be a \emph{contractive mapping} on the metric space $(\calQ,\norm{\mdot}$), i.e.,  $\norm{D_q (E_q(\bfq_1)) - D_q (E_q(\bfq_2))} \leq L \norm{\bfq_1-\bfq_2}$ for some $0\leq L< 1$ and any $\bfq_1,\bfq_2 \in \calQ$.
Then $\norm{\widehat \bfq - \bfq}\leq \frac{1}{1-L} (\| \widehat \bfq - D_q(E_q(\widehat \bfq))\| + \|  D_q(E_q(\bfq)) - \bfq\|  )$.
\end{proposition}

\begin{proof}
By triangular inequality we get
\begin{align*}
\norm{\widehat \bfq - \bfq} & = \norm{\widehat \bfq
- D_q(E_q(\widehat \bfq)) + D_q(E_q(\widehat \bfq)) - D_q(E_q(\bfq)) + D_q(E_q(\bfq)) - \bfq
} \\
& \leq \norm{\widehat \bfq
- D_q(E_q(\widehat \bfq))} + \norm{D_q(E_q(\widehat \bfq)) - D_q(E_q(\bfq))} + \norm{D_q(E_q(\bfq)) - \bfq}\\
& \leq \norm{\widehat \bfq
- D_q(E_q(\widehat \bfq))} + L\norm{\widehat \bfq - \bfq} + \norm{D_q(E_q(\bfq)) - \bfq}.
\end{align*}
Therefore
\begin{align*}
\norm{\widehat \bfq - \bfq}
& \leq \tfrac{1}{1-L}\left(\norm{\widehat \bfq
- D_q(E_q(\widehat \bfq))} + \norm{D_q(E_q(\bfq)) - \bfq}\right).
\end{align*}
\end{proof}

Similar to the previous proposition, if we assume that the autoencoder performs well for any $\bfq \in \calQ$, we can conclude that the error $\norm{\widehat \bfq - \bfq}$ is small. Notice that the first term in the bound is computable and is exactly the RMA metric in \Cref{{eq:metrics}}, up to a normalizing constant, that we used to detect OOD samples.
Thus, if the RMA is large, we cannot expect the overall model error to be small, and further updates or refinements are warranted.  If, however, the RMA is small and we expect $\norm{\bfq - D_q(E_q(\bfq))}$ to be small and $L$ is not too close to $1$, then we can expect the error to be small.

We can also get a bound using assumptions on the continuity and error of the autoencoders, as well as the forward and inverse surrogates.

\begin{theorem}
\label{thm:error}
    Let $E_b:\calB \to \calZ_b$ and $D_q:\calZ_q \to \calQ$ be Lipschitz continuous on their corresponding metric spaces with Lipschitz constants $L_b, L_q \geq0$ each equipped with sub-multiplicative norms and $\norm{\bar\bfb-\bfb} \leq \delta$ for $\delta \geq 0$. Further, assume
    for any $\bfz_b \in \calZ_b$ $\bfz_q \in \calZ_q$, and $\bfq \in \calQ$ there exist constants $\xi_b,\xi_M,\xi_q \geq 0$ such that
    $$\norm{E_b(D_b(\bfz_b)) - \bfz_b} \leq \xi_b, \qquad \| \bfM^\dagger \bfM \bfz_q - \bfz_q \| \leq \xi_M,\qquad\text{and}\qquad\norm{D_q(E_q(\bfq)) - \bfq} \leq \xi_q.$$
    Then
    \begin{equation*}
        \norm{\widehat\bfq-\bfq} \leq L_q \left(\norm{\bfM^\dagger}  \left(L_b \delta +  \xi_b\right)+\xi_M \right) + \xi_q.
    \end{equation*}

\end{theorem}

\begin{proof}
Using the definition of $\widehat \bfq$ and the triangular inequality, we have
\begin{align*}
\norm{\widehat \bfq - \bfq} & = \norm{D_q(\bfM^\dagger E_b(\bfb)) - D_q(E_q(\bfq))+D_q(E_q(\bfq))- \bfq} \\
& \leq L_q \norm{\bfM^\dagger E_b(\bfb) - E_q(\bfq)} + \norm{D_q(E_q(\bfq))- \bfq} \leq L_q \norm{\bfM^\dagger E_b(\bfb) - E_q(\bfq)} + \xi_q.
\end{align*}

We also have the following bound in the latent space,
\begin{align*}
\norm{\bfM^\dagger E_b(\bfb) - E_q(\bfq)} & = \norm{\bfM^\dagger E_b(\bfb) - \bfM^\dagger \bfM E_q(\bfq) + \bfM^\dagger \bfM E_q(\bfq) - E_q(\bfq)} \\
& \leq \norm{\bfM^\dagger E_b(\bfb) - \bfM^\dagger \bfM E_q(\bfq) } + \norm{\bfM^\dagger \bfM E_q(\bfq)  - E_q(\bfq)} \\
& \leq \norm{\bfM^\dagger} \norm{E_b(\bfb) - \bfM E_q(\bfq)} + \xi_M
\end{align*}
where since
\begin{align*}
 \norm{E_b(\bfb) - \bfM E_q(\bfq)} & = \norm{E_b(\bfb) -E_b(\bar\bfb) + E_b(\bar\bfb)- \bfM E_q(\bfq)} \\
& \leq \norm{E_b(\bfb) -E_b(\bar\bfb)} + \norm{E_b(\bar\bfb)- \bfM E_q(\bfq)}\\
& \leq L_b \delta + \norm{E_b(D_b(\bfM E_q(\bfq))- \bfM E_q(\bfq)}\\
& \leq L_b \delta + \xi_b,
\end{align*}
we get
\begin{equation*}
\norm{\widehat \bfq - \bfq} \leq L_q\left(\norm{\bfM^\dagger} (L_b \delta + \xi_b) + \xi_M \right) + \xi_q.
\end{equation*}
\end{proof}

Notice that this bound incorporates all learned mappings.  We observe that in addition to the Lipschitz constants, the bound relies on the accuracy of the model autoencoder $\xi_q$, the data autoencoder in the latent space $\xi_b$, and the invertibility of the latent mappings $\xi_M$. Moreover, the value of $\delta$ relies on the accuracy of the forward surrogate and the noise in the data.
In the special case where $\bfM^\dagger \bfM = \bfI$, then $\xi_M$ vanishes.  Similarly, if the model and data autoencoders compress their inputs well, then $\xi_q$ and $\xi_b$ are small.
Overall, these bounds distinguish our method from \cite{feng2023simplifying, kun2015coupled}, in that we do not treat $D_b$ and $E_q$ as auxiliary.

\subsection{Inverse electromagnetic problem}
\label{sub:electro}
Another challenging inverse problem is to estimate the presence of water or minerals in the Inverse Electromagnetic Problem (see \cite{haberBook2014} for details). For this problem, the forward problem is Maxwell's equations,
\begin{eqnarray}
    \label{maxwell}
    B_t = -\nabla \times E \quad \quad \nabla \times (\mu^{-1} B) - \sigma E = J.
\end{eqnarray}
Here $E$ is the electric field, $B$ is the magnetic field, $J$ is the source term, $\mu$, is the magnetic susceptibility, and $\sigma$ is the conductivity. The equations are equipped with appropriate initial and boundary conditions. For the problem we solve here, we use a common data collection approach where the magnetic field, $B(t)$, is recorded above the surface of the earth, usually by flying. The goal is to recover the conductivity $\sigma(\bfx)$ given the data $B$.
The problem is discretized on a staggered grid and the forward problem can be written as (see \cite{haberBook2014} for details)
\begin{eqnarray}
    \label{discmax}
    \bfB_t = - {\bfC} \bfM((\mu\sigma)^{-1}) \bfC^{\top} \bfB, \quad \quad \bfB(0) = \bfB_0.
\end{eqnarray}
Here, $\bfC$ is the discretization of the $\nabla \times$ operator and $\bfM$ is a mass matrix.
The backward Euler method is used to integrate the system in time obtaining the simulated data.
The problem is nonlinear with respect to $\sigma$, and the solution of each forward problem requires the solution of $n$ linear systems of equations, where $n$ is the number of time steps.
To obtain a data set we use geo-statistical models.
A few of these models, including the reconstructed solutions using our trained paired autoencoder networks are provided in \Cref{aem}.
\begin{figure}[h]
\begin{tabular}{ccc}
    \includegraphics[width=4.5cm]{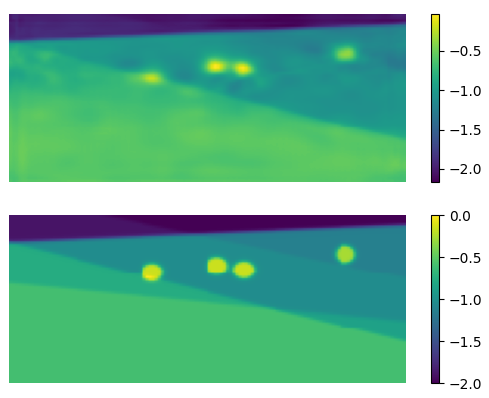}
 &
    \includegraphics[width=4.5cm]{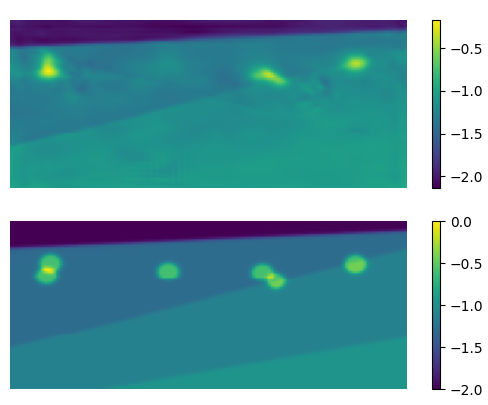}
 &
    \includegraphics[width=4.5cm]{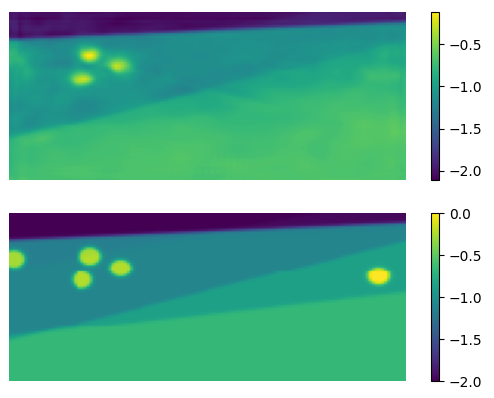}
\end{tabular}
\caption{Three models from the geo-statistical distribution (bottom) and their recoveries (top).} \label{aem}
\end{figure}

Inference results are presented in \Cref{tab:em_quant_results}. Once again we see that the LFE yields a reasonable recovery that approximately fits the data.  Moreover, by combining a ``warm start'' approach with a few steps of LSI, we can obtain models with improved accuracy and outperform other methods.

\begin{table}[h]
    \caption{Means of the data misfit and model errors for the electromagnetic inverse problem, calculated over $50$ examples. All errors are in terms of relative $\ell^2$ error.\\}
    \label{tab:em_quant_results}
    \centering
    \begin{tabular}{llcccc}
    \hline
        &\multirow{2}{*}{\textbf{Initial Guess}}&\multicolumn{2}{c}{\textbf{Data Misfit}} & \multicolumn{2}{c}{\textbf{Model Error}} \\
        &  & Initial & Final & Initial & Final \\      \hline
        \multirow{2}{*}{BI}& basic start, $\bfq_0^{\rm basic}$  & 0.922 & 0.045 & 1.83 & 0.101\\
        &warm start, $\widehat\bfq$   & 0.092 & 0.041 & 0.143 & 0.104\\ \hline
        \multirow{2}{*}{LSI}&basic start, $\bfz_{\bfq_0^{\rm basic}}$, $\alpha=0$ & 0.921 & 0.043 & 1.65 & 0.092\\
        &warm start, $\widehat\bfz, \alpha = 1 $ (ours) & 0.092 & 0.039 & 0.143 & \textbf{0.082}\\ \hline
    \end{tabular}
\end{table}

\subsection{Neural network details}\label{sub:seismicdetails}
Here, we provide the network designs used for the numerical experiments.

\textbf{Seismic inversion.} For this experiment we use a shared latent representation for the data and the model, i.e., $\bfM  = \bfM^\dagger =\bfI$. The encoders and decoders are all characterized by ResNet \cite{he2016deep} blocks interleaved with average pooling.

Our ResNet block design is given by
\begin{equation}\label{eq:resnet_block}
    \tX_{i+1} = \tX_i - h \tK_1 \sigma(N(\tK_2 \tX_i))
\end{equation}
where $i$ is the layer index, $\tX$ the network state, $\sigma(\cdot)$ is the SiLU nonlinear activation function, $N(\cdot)$ is the batch normalization operation, and $\tK_1$ and $\tK_2$ correspond to two different learnable convolution operators with $3 \times 3$ kernels while preserving the number of channels. Each level in the multi-level encoder or decoder then follows as
\begin{equation}\label{eq:resnet_block_full}
    \tX_{i+1} = \tK_{cc} \tX_i \rightarrow 3\times \text{ResNet block (\Cref{eq:resnet_block})} \rightarrow \tX_{i+1} = P(\tX_i),
\end{equation}
where $P(\cdot)$ is the average pooling operation that halves the number of elements in each dimension and $\tK_{cc}$ is a learnable convolution operator that changes the number of channels. Similar to the ResNet block with downsampling, we construct
\begin{equation}\label{eq:resnet_block_full2}
    \tX_{i+1} = U(\tX_i) \rightarrow \tX_{i+1} = \tK_{cc} \tX_i \rightarrow 3\times \text{ResNet block (\Cref{eq:resnet_block})},
\end{equation}
where $U(\cdot)$ is $2\times$ upsampling operator.

\Cref{tab:seismic_data_encoder} provides the full four-level network design for the seismic data encoder. The final layer is the only one that is not convolutional and instead uses an affine map represented by the linear operator $\tA$. Similarly, the seismic data decoder (\Cref{tab:seismic_data_decoder}) starts with a different learnable linear operator.

The encoder and decoder for the velocity model, $E_q$ and $D_q$, respectively, are the same as described above and in \Cref{tab:seismic_data_encoder} and \Cref{tab:seismic_data_decoder}, except that the input is a single channel of size $n_z \times n_x$ pixels. The dimensions of the linear operators $\tA$ change accordingly.

\begin{table}
    \centering
    \begin{tabular}{lcl}
        \textbf{Layer} & \textbf{Feature size} & \textbf{Type}\\
        \hline
         input & $30 \times 72 \times 144$ & channels $\times$ receivers $\times$ time samples\\
         1-3 & $32 \times 72 \times 144$ & \Cref{eq:resnet_block_full} \\
         4-6 & $32 \times 36 \times 72$ & \Cref{eq:resnet_block_full} \\
         7-9 & $64 \times 18 \times 36$ & \Cref{eq:resnet_block_full}\\
         10-12 & $64 \times 9 \times 18$ & \Cref{eq:resnet_block} $\times 3$\\
         13 & $128 \times 9 \times 18$ & $ \tX_{i+1} = \tK_{cc} \tX_i$\\
         14 & $512$ & $\tX_{i+1} = N(\tA\vec{\tX_i})$\\
    \end{tabular}
    \caption{Network for the seismic data encoder $E_b(\bfb)$}
    \label{tab:seismic_data_encoder}
\end{table}

\begin{table}
    \centering
    \begin{tabular}{lcl}
        \textbf{Layer} & \textbf{Feature size} & \textbf{Type}\\
        \hline
         input & $512$ & Latent variables\\
         1 & $128 \times 9 \times 18$ & $\tX_{i+1} = \tA\vec{\tX_i} \rightarrow$ reshape\\
         2 & $64 \times 9 \times 18$ & $ \tX_{i+1} = \tK_{cc} \tX_i$\\
         3-5 & $64 \times 9 \times 18$ & \Cref{eq:resnet_block} $\times 3$\\
         6-8 & $32 \times 18 \times 36$ & \Cref{eq:resnet_block_full2}\\
         9-11 & $16 \times 36 \times 72$ & \Cref{eq:resnet_block_full2} \\
         12-14 & $16 \times 72 \times 144$ & \Cref{eq:resnet_block_full2} \\
         15 & $30 \times 72 \times 144$ & $ \tX_{i+1} = \tK_{cc} \tX_i$\\
     \end{tabular}
    \caption{Network for the seismic data encoder $D_b(\bfb)$}
    \label{tab:seismic_data_decoder}
\end{table}

\end{document}